\documentclass{article}
\usepackage{ijcai20}

\usepackage{times}  							
\usepackage{soul}
\usepackage{url}
\usepackage[hidelinks]{hyperref}
\usepackage[utf8]{inputenc}
\usepackage[small]{caption}
\usepackage{graphicx}  						 
\usepackage{xcolor}
\usepackage{xspace}
\usepackage{amsmath}
\usepackage{amssymb}
\usepackage{amsfonts}
\usepackage[lined,ruled,linesnumbered,shortend]{algorithm2e}
\usepackage{multirow}
\usepackage{enumitem}
\newcommand{\FOL}{\mathbf{FO}}

\newcommand{\myi}{(\emph{i})\xspace}
\newcommand{\myii}{(\emph{ii})\xspace}

\newcommand{\A}{\mathcal{A}} \newcommand{\B}{\mathcal{B}}
\newcommand{\C}{\mathcal{C}}

\newcommand{\I}{\mathcal{I}} 
 \renewcommand{\L}{\mathcal{L}}
 
\renewcommand{\O}{\mathcal{O}} \renewcommand{\P}{\mathcal{P}}
\newcommand{\Q}{\mathcal{Q}} \newcommand{\R}{\mathcal{R}}
\renewcommand{\S}{\mathcal{S}} \newcommand{\T}{\mathcal{T}}
 \newcommand{\V}{\mathcal{V}}

\newcommand{\wrt}{w.r.t.\xspace}

\newcommand{\per}{\mbox{\bf .}}                  % period

\newcommand{\tup}[1]{\langle #1\rangle}            % tuple

\newcommand{\dllite}{\textit{DL-Lite}\xspace}
\newcommand{\dlliter}{\textit{DL-Lite}_{\R}\xspace}

\newcommand{\dlliterden}{\textit{DL-Lite}_{\R,den}\xspace}
\newcommand{\ACz}{\textsc{AC}\ensuremath{^0}\xspace}

\newcommand{\NOT}{\neg}

\newcommand{\SOMET}[1]{\exists #1}

\newcommand{\INV}[1]{#1^{-}}

\newcommand{\ISA}{\sqsubseteq}

\newcommand{\unify}{\textsf{Saturate}}

\newcommand{\aczero}{\ACz}

\def\qedfull{\hfill{\qedboxfull}   % qed with full box
  \ifdim\lastskip<\medskipamount \removelastskip\penalty55\medskip\fi}
\def\qedboxfull{\vrule height 4pt width 4pt depth 0pt}

\newcommand{\OWLTWO}{\textsc{owl\,2}\xspace}
\newcommand{\OWLQL}{\textsc{owl\,2\,ql}\xspace}

\newcommand{\cll}[2]{\mathsf{cl_{#1}^{\T}(#2)}}

\newcommand{\cens}[1]{\gamma(#1)}

\newcommand{\GA}{\mathbf{GA}}

\newcommand{\CQL}{\mathbf{CQ}}

\newcommand{\eval}{\textit{Eval}}
\newcommand{\lc}{\L}

\renewcommand{\cens}{\mathsf{cens}}

\newcommand{\pol}{\P}

\definecolor{notecolor}{rgb}{0.4,0,0}
\newsavebox{\fmbox}

\newtheorem{theorem}{Theorem}
\newtheorem{corollary}{Corollary}

\newtheorem{lemma}{Lemma}

\newtheorem{definitionAux}{Definition}
\newenvironment{definition}{\begin{definitionAux}\rm}{\end{definitionAux}}

\newtheorem{claimAux}{Claim}
\newenvironment{claim}{\begin{claimAux}\rm}{\end{claimAux}}

\newtheorem{exampleAux}{Example}
\newenvironment{example}{\begin{exampleAux}\rm}{\end{exampleAux}}

\newtheorem{examplesAux}{Examples}

\newtheorem{constructionAux}{Construction}

\newenvironment{proof}{\noindent \textsl{Proof.\ }}{\qedfull}

\long\def\eatpar#1{%
\ifx#1\par                      % se il token e' \par
\let\nextmove=\eatpar           % rimetti \eatpar in coda
\else
\let\nextmove=#1%               altrimenti, rimetti il token in coda
\fi
\nextmove%                      il token o \eatpar viene rimesso in coda
}

\def\qed{\hfill{\qedboxempty}      % qed with empty box
  \ifdim\lastskip<\medskipamount \removelastskip\penalty55\medskip\fi}

\def\qedboxempty{\vbox{\hrule\hbox{\vrule\kern3pt
                 \vbox{\kern3pt\kern3pt}\kern3pt\vrule}\hrule}}

\def\qedfull{\hfill{\qedboxfull}   % qed with full box
  \ifdim\lastskip<\medskipamount \removelastskip\penalty55\medskip\fi}

\def\qedboxfull{\vrule height 4pt width 4pt depth 0pt}

\newcommand{{\incolumn}}[1]{\begin{tabular}[c]{c} #1 \end{tabular}}
\newcommand{{\incolumnmath}}[1]{\begin{array}[c]{c} #1 \end{array}}
	\title{CQE in Description Logics Through Instance Indistinguishability \\ (extended version)}
\author{
	Gianluca Cima$^1$ \and
	Domenico Lembo$^1$ \and
	Riccardo Rosati$^1$ \And
	Domenico Fabio Savo$^{2}$\\
	\affiliations
	$^1$Sapienza Università di Roma\\
	$^2$Università degli Studi di Bergamo
	\emails
	\{cima, lembo, rosati\}@diag.uniroma1.it,
	domenicofabio.savo@unibg.it
}

\begin{document}	
	
	\maketitle
	
\begin{abstract}
  We study privacy-preserving query answering in Description Logics (DLs). Specifically, we consider the approach of controlled query evaluation (CQE) based on the notion of \emph{instance indistinguishability}. We derive data complexity results for query answering over $\dlliter$ ontologies, through a comparison with an alternative, existing confidentiality-preserving approach to CQE.
  Finally, we identify a semantically well-founded notion of approximated query answering for CQE, and prove that, for $\dlliter$ ontologies, this form of CQE is tractable with respect to data complexity and is first-order rewritable, i.e., it is always reducible to the evaluation of a first-order query over the data instance.  
\end{abstract}

	\section{Introduction}
\label{sec:introduction}

\newcommand{\proja}{\mathsf{ProjA}}
\newcommand{\projb}{\mathsf{ProjB}}
\newcommand{\supplier}{\mathsf{Supplier}}

%We deal with the problem of controlled query evaluation (CQE), a.k.a.\ privacy-preserving query answering~\cite{SiJR83,BonattiKS95,Biskup00}.  %BiBo04,

%The declarative approaches to this problem that have been proposed in the literature in knowledge representation and database theory are based on the idea of defining a \emph{data protection policy} through logical statements: e.g.,  a denial assertion of the form
%\[
%\forall x \per \, \supplier(x) \land \proja(x) \land \projb(x) \rightarrow \bot
%\]
%would express that the system should not let the users know the suppliers involved in both Project A and Project B.

%%The declarative proposals to CQE formalize the notion of privacy preservation based on two different approaches.
%%in two distinct ways.
%CQE proposals formalize the notion of privacy preservation based on two different approaches.

We consider controlled query evaluation (CQE), a declarative framework for privacy-preserving query answering investigated in the literature on knowledge representation and database theory~\cite{SiJR83,BonattiKS95,Biskup00}.  %BiBo04,
The basic idea of CQE is defining a \emph{data protection policy} through logical statements. Consider for instance an organization that wants to keep confidential the fact that it has suppliers involved in both Project A and Project B. This can be expressed over the information schema of the organization through a denial assertion of the form
\[
\forall x \per \, \supplier(x) \land \proja(x) \land \projb(x) \rightarrow \bot
\]

In CQE, two different main approaches %formalizing the notion of privacy preservation 
can be identified.
The first one \cite{BiskupB04,BiBo04,BiskupW08,BeCK18,BBJT19,StWe14} models privacy preservation through the notion of \emph{indistinguishable data instances}. In this approach, a system for CQE enforces data privacy %is privacy-preserving 
if, for every data instance $I$, there exists a data instance $I'$ that does not violate the data protection policy and is indistinguishable from $I$ for the user, i.e., for every user query $q$, the system provides the same answers to $q$ over $I$ and over $I'$. We call this approach \emph{(instance) indistinguishability-based} (IB). 
In continuation of the previous example, in the presence of an instance $\{\supplier(c), \proja(c), \projb(c)\}$, an IB system should answer user queries as if the instance were, e.g., $\{\supplier(c),\proja(c)\}$ (note that other instances not violating the policy can be considered as indistinguishable, e.g., $\{\supplier(c),\projb(c)\}$). 

The second approach~\cite{BoSa13,GKKZ13,GKKZ15} models privacy preservation by considering the whole (possibly infinite) set of answers to queries that the system provides to the user. In this approach, a CQE system protects the data %is privacy-preserving 
if, for every data instance $I$, the logical theory corresponding to the set of answers provided by the system to all queries over $I$ does not entail any violation of the data protection policy. 
%%
%This approach has been recently extended and parametrized, introducing a so-called \emph{censor language} that may be different from the user query language~\cite{LeRS19}. 
%%
According to~\cite{GKKZ15}, we call this approach \emph{confidentiality-preserving (CP)}.
In our ongoing example, a CP system
%% using conjunctive queries as both user and censor language
would entail, e.g., the queries $\supplier(c) \land \proja(c)$ and $\exists x \per \supplier(x) \land \projb(x)$, but not also to the query $\supplier(c) \land\projb(c)$ (notice that the choice is non-deterministic, and in our example the system could have decided to disclose that $c$ participates in Project B and hide its participation in Project A).
%%
%% This approach has been recently extended and parameterized, introducing a \emph{censor language}, which is the language used to specify the logical formulas %among those implied by the system ignoring the policy that can be disclosed to the users without violating the policy, and which may be different from the user query language~\cite{LeRS19}.

%In both approaches, only \emph{optimal} CQE systems are considered, i.e., those who are the most informative for the given data instance $I$ among all the privacy-preserving systems (in other words, those who return a maximal subset of the query answers that are returned over $I$ regardless of privacy preservation).

In both approaches, the ultimate goal is to realize \emph{optimal} CQE systems, i.e., systems maximizing the answers returned to user queries, still respecting the data protection policy. Traditionally, this aim has been pursued through the construction of a \emph{single} \emph{optimal censor}, i.e., a specific implementation of the adopted notion of privacy-preservation, either IB or CP. Since, however, in both approaches several optimal censors typically exist, this way of proceeding requires to make a choice on how to obfuscate data, which, in the absence of additional (preference) criteria, may result discretionary. To avoid this, query answering over all optimal censors has been recently studied (limited to the CP approach)~\cite{GKKZ13,LeRS19}.
%
% NON SO SE SI PUò INTUIRE IL PERCHè DI QUESTA IMPLICAZIONE
%
%It is not hard to derive that, under general assumptions, privacy preservation in the first approach implies the notion of privacy preservation in the second one, but not vice versa. We thus call \emph{strict} the approach to CQE based on instance indistinguishability. However, the precise relationship between these two notions is not clear and has not been fully investigated yet
%

Despite their similarities,
%% and common aim,
the precise relationship between the IB and CP approaches is still not clear and has not been fully investigated yet.
Also, query answering over all optimal IB censors has not been previously studied. 
%considered in the literature.
%%
Moreover, among the complexity results obtained and the techniques defined so far for CQE, we still miss the identification of cases that are promising towards its practical usage. %of CQE.

%In this paper, we focus on the application to Description Logic ontologies of the approach to CQE based on instance indistinguishability (Section \ref{sec:framework}).
%%
%We first study the relationship between the above two approaches to CQE (Section \ref{sec:CP-censors}): in particular, we prove that the IB approach to CQE in DLs corresponds to a particular instance of the parameterized CP approach to CQE of \cite{LeRS19}.

In this paper, we aim at filling some of the above mentioned gaps in the context of Description Logic (DL) ontologies.\footnote{Privacy-preserving query answering in DLs has been investigated also in settings different from CQE: see, e.g., \cite{GrHo08,CDLR12,TaSH14}.} We focus on the approach to CQE based on instance indistinguishability (Section \ref{sec:framework}), and study its relationship with the CP approach %to CQE 
(Section \ref{sec:CP-censors}). Specifically, we prove that the IB approach to CQE in DLs corresponds to a particular instance of the CP approach to CQE~\cite{LeRS19}.
Based on such a correspondence, for ontologies specified in the well-known DL $\dlliter$~\cite{CDLLR07}, we are able to transfer some complexity results for query answering over all optimal censors shown in~\cite{LeRS19} %to the case of IB CQE
to the case of CQE under IB censors (Section \ref{sec:cqe-over-strict}).
Such results show that, even in the lightweight DL $\dlliter$, query answering in the IB approach is intractable with respect to data complexity, unless one relies on a single optimal censor chosen non-deterministically in the lack of further meta-information about the domain of the dataset.

To overcome the above problems and provide a practical, semantically well-founded solution, we define a \emph{quasi-optimal} notion of IB censor, which corresponds to the best sound approximation of all the optimal IB censors (Section \ref{sec:cqe-approximation}).
We then prove that, in the case of $\dlliter$ ontologies, query answering based on the quasi-optimal IB censor is tractable with respect to data complexity and is reducible to the evaluation of a first-order query over the data instance, i.e., it is \emph{first-order rewritable}. %This property opens the possibility of defining algorithms for CQE of practical usage.
We believe that this result has an important practical impact. Indeed, we have identified a setting in which privacy-preserving query answering formalized in a declarative logic-based framework as CQE, for a DL (i.e., $\dlliter$) specifically designed for data management, has the same data complexity as evaluating queries over a database (i.e., $\ACz$). This opens the possibility of defining algorithms for CQE of practical usage, amenable to implementation on top of traditional (relational) data management systems, as in Ontology-based Data Access~\cite{XiaoCKLPRZ18}.
%The long version of the paper providing the complete proofs is available at~\cite{ArXive}.
%\iflong
%\else
%The long version of the paper providing the complete proofs is available at \url{http://bit.ly/paper6887}.
%\fi
%%\nb{Aggiunto link alla versione lunga dell'articolo}

%%% Local Variables:
%%% mode: latex
%%% TeX-master: "main"
%%% save-place: t
%%% End:

	\section{Preliminaries}\label{sec:preliminaries}
We use standard notions of function-free first-order (FO) logic, and in particular we consider
Description Logics (DLs), which are fragments of FO using only unary and binary predicates, called concepts and roles, respectively~\cite{BCMNP07}. %,BHLS17
We assume to have the pairwise disjoint countably infinite sets $\Sigma_C, \Sigma_R, \Sigma_I$ and $\Sigma_{\V}$ for \emph{atomic concepts}, \emph{atomic roles}, \emph{constants} (a.k.a.\ individuals), and \emph{variables}, respectively.
A DL ontology $\O=\T \cup \A$ is constituted by a TBox $\T$ and an ABox $\A$, specifying intensional and extensional knowledge, respectively. The set of atomic concepts and roles occurring in $\O$ is the \emph{signature} of $\O$. The semantics of $\O$ is given in terms of FO models over the signature of $\O$, in the standard way~\cite{BCMNP07}. 
%We denote with $\Mod{\O}$ the set of models of a DL ontology $\O$. If $\Mod{\O}\neq \emptyset$, %
In particular, we say that $\O$
is \emph{consistent} if it has at least one model, \emph{inconsistent} otherwise. $\O$ \emph{entails} an FO sentence $\phi$ specified over the signature of $\O$, denoted $\O \models \phi$, if $\phi$ is true in every model of $\O$. %$\I \in \Mod{\O}$.
%%%%%%%%%%%%
In this paper, we consider ontologies expressed 
in $\dlliter$, the member of the $\dllite$ family~\cite{CDLLR07} which underpins $\OWLQL$~\cite{W3Crec-OWL2-Profiles}, i.e., the $\OWLTWO$ profile specifically 
designed for efficient query answering. A TBox $\T$ in $\dlliter$ is a finite set of axioms of the form $B_1 \ISA B_2$ (resp., $R_1 \ISA R_2$), denoting concept (resp., role) inclusion, and $B_1 \ISA \neg B_2$ (resp., $R_1 \ISA \neg R_2$), denoting concept (resp., role) disjointness, 
where: $R_1, R_2$ are 
%\emph{basic roles}, i.e., 
%expressions 
of the form 
$P$, with $P \in \Sigma_R$, or its inverse $\INV{P}$,  and $B_1, B_2$ are 
%\emph{basic concepts}, i.e., expressions 
of the form $A$, with $A \in \Sigma_C$, $\SOMET{P}$, or $\SOMET{\INV{P}}$, i.e., unqualified existential restrictions, which denote the set of objects occurring as first or second argument of $P$, respectively.  
An ABox $\A$ %for $\T$ 
is a finite set of \emph{ground atoms}, i.e., 
assertions of the form $A(a)$, $P(a,b)$, where $A \in \Sigma_C$, $P \in \Sigma_R$, and $a,b \in \Sigma_I$.
%%
%%
%The semantics of a $\dlliter$ ontology $\O$ is given in terms of FO models over the signatore of $\O$ in the standard way~\cite{CDLLR07}. 
%%We denote with $\Mod{\O}$ the set of models of a $\dlliter$ ontology $\O$. If $\Mod{\O}\neq \emptyset$, %
%We say that a %$\dlliter$ ontology 
%$\O$
%is \emph{consistent} if it has at leats one model, \emph{inconsistent} otherwise. $\O$ \emph{entails} an FO sentence $\phi$, denoted $\O \models \phi$, if $\phi$ is true in every model of $\O$. %$\I \in \Mod{\O}$.
As usual in query answering over DL ontologies, we focus on the language of conjunctive queries. 
A Boolean conjunctive query (BCQ) $q$ is an FO sentence of the form $\exists \vec{x} \per \phi(\vec{x})$, where $\vec{x}$ are variables in $\Sigma_{\V}$, and $\phi(\vec{x})$ is a finite, non-empty conjunction of atoms of the form $\alpha(\vec{t})$, where $\alpha \in \Sigma_C \cup \Sigma_R$, and each term in $\vec{t}$ is either a constant in $\Sigma_I$ or a variable in $\vec{x}$. 
We denote by $\eval(q,\A)$ the evaluation of a query $q$ over (the model isomorphic to) an ABox $\A$.

A \emph{denial assertion} (or simply a denial) is an FO sentence of the form $\forall \vec{x}. \phi(\vec{x}) \rightarrow \bot$,  such that $\exists \vec{x}. \phi(\vec{x})$ is a BCQ. Given one such denial $\delta$ and an ontology $\O$, we say that  $\O \cup \{\delta\}$ is consistent if $\O \not \models \exists \vec{x}. \phi(\vec{x})$, and is inconsistent otherwise.
%An ontology $\O$ satisfies a denial if $\O \not \models \exists \vec{x}. \phi(\vec{x})$\nb{Serve? credo aiuti per la parte CQA, ma e' necessario/sufficiente (o va cambiato)?}.

In the following, with $\FOL$, $\CQL$, and $\GA$  we denote the languages of function-free FO sentences, BCQs, and ground atoms, respectively, all specified over the alphabets $\Sigma_C, \Sigma_R, \Sigma_I$, and $\Sigma_{\V}$. 
%($\GA$ however does not involve variables in $\Sigma_{\V}$).
%
Given an ontology $\O$ and a language $\L$, with $\L(\O)$ we refer to the subset of $\L$ whose sentences are built over the signature 
of $\O$ and the variables in $\Sigma_{\V}$. %
For a TBox $\T$ and a language $\L$, %$\L \subseteq \FOL$, 
we denote by $\cll{\L}{\cdot}$ the function that, for an ABox $\A$, returns all the sentences $\phi \in \L(\T \cup \A)$ such that $\T \cup \A \models \phi$. 

For the sake of presentation, we will limit our technical treatment to languages containing only closed formulas, but our results hold also for open formulas. In particular, the results on entailment of BCQs (see Sections \ref{sec:cqe-over-strict} and~\ref{sec:cqe-approximation})  can be extended to arbitrary (i.e., non-Boolean) CQs in the standard way\footnote{It is also easy to see that, since $\dlliter$ is insensitive to the adoption of the \emph{unique name assumption} (UNA) for CQ answering~\cite{ACKZ09}, our results hold both with and without UNA.}. 
Our complexity results are for data complexity, i.e., are w.r.t.\ the size of the ABox only.

%%% Local Variables:
%%% mode: latex
%%% TeX-master: "main"
%%% save-place: t
%%% End:

	%\section{CQE Framework}
\section{CQE through instance indistinguishability}
\label{sec:framework}

\newcommand{\pa}{\mathit{P_A}}
\newcommand{\pb}{\mathit{P_B}}

\newcommand{\stcens}{\mathsf{ib\_cens}}
\newcommand{\inters}{pat}
\newcommand{\pvcens}{\mathsf{cens^{pv}}}
\newcommand{\get}{\mathsf{get}}
\newcommand{\patient}{\mathsf{patient}}
\newcommand{\cpcenset}[2]{\mathsf{#1\text{-}OptCPCens}_{#2}}
\newcommand{\stcenset}[1]{\mathsf{OptIBCens}_{#1}}
\newcommand{\censet}[1]{\mathsf{OptCPCens}_{#1}}
\newcommand{\optstcensA}{\cens_{1}}
\newcommand{\optstcensB}{\cens_{2}}
\newcommand{\intcens}{\cens_{3}}
\newcommand{\optcpcensA}{\cens_{4}}

A CQE framework consists of a TBox $\T$ and a policy $\P$ over $\T$, i.e., a finite set of denial assertions over the signature of $\T$. An ABox $\A$ for $\T$ is such that $\A$ and $\T$ have the same signature. In the following, when a TBox $\T$ is given, we always assume that the coupled policy is specified over $\T$, that each considered ABox $\A$ is for $\T$, and that, unless otherwise specified, $\T \cup \A$ and $\T \cup \P$ are consistent.
%
%A CQE framework consists of a TBox $\T$ and a policy $\P$ over $\T$ (or simply a policy), i.e., a finite set of denial assertions over the signature of $\T$.
%Data are stored in an ABox $\A$, which we always assume to have the same signature of $\T$ and be such that $\T \cup \A$ is consistent\nb{e' sufficiente per evitare di dire sempre a policy over $\T$ e ABox for $\T$?}. 
A \emph{censor} is a function that alters query answers so that by uniting the answers (even a possibly infinite set thereof) with the TBox a user can never infer a BCQ $\exists \vec{x} \per \phi(\vec{x})$, for each denial $\forall \vec{x} \per \phi(\vec{x}) \rightarrow \bot$ in $\P$.

We here propose a notion of censor which is the natural application to our framework of the analogous definitions given in~\cite{BiskupB04,BiskupW08,BeCK18,BBJT19}.
%} and recently considered in \cite{BeCK18,BBJT19}. %in the context of ontology-based data integration. 
The basic idea of this approach is that for every underlying instance (an ABox in our framework) and every query, a censor returns to the user the same answers it would return on another (possibly identical) instance that does not contain confidential data, so that she cannot understand which of the two instances she is querying. This is formalized as follows.

%For guaranteeing the complete protection of the secrets in an ABox $\A$, in~\cite{BiskupB04,BiskupW08} the authors propose a notion of censor which requires the existence of an ABox $\A'$ that does not contain secrets and that, respect to the censor, is \textit{indistinguishable} from the actual ABox $\A$. In this way, whatever the actual ABox is, a malicious user cannot determinate whether $\A$ or $\A'$ is the actual ABox and, so, she/he cannot decide if a secret is true or not in the ABox at hand. 

\begin{definition}[Indistinguishability-based censor]
\label{def:strict-censor}
	Let $\T$ be a DL TBox and $\pol$ be a policy. An \emph{indistinguishability-based (IB) censor}
	%\footnote{We call such censors ``strict'' since, as we will show in the next section, they also satisfy another well-known notion of censor proposed in literature~\cite{GKKZ15}, which however define as censors also some functions ruled out by Definition~\ref{def:strict-censor}.} 
	for $\T$ and $\P$ is a function  $\cens(\cdot)$ %$\stcens_{\T,\pol}$ 
	that, for each ABox $\A$, returns a set $\cens(\A) \subseteq \cll{\CQL}{\A}$ %$\stcens_{\T,\pol}(\A) \subseteq \cll{\CQL}{\A}$ 
	such that there exists an ABox $\A'$ for which $(i)$ %$\stcens_{\T,\pol}(\A) = \stcens_{\T,\pol}(\A')$ 
	$\cens(\A) = \cens(\A')$
	(in this case we say that $\A$ and $\A'$ are \emph{indistinguishable} w.r.t.\ $\cens$) %$\stcens_{\T,\pol}$), 
	and $(ii)$ $\T \cup \P \cup \A'$ is a consistent FO theory.
\end{definition}

\begin{example}\label{ex:optimal-censor}
	%Consider a scenario in which a company does not want to reveal which among its suppliers is involved in both projects $\pa$ and $\pb$. 
	Let us now formalize more precisely the scenario we have used for the examples in the introduction, by instantiating our CQE framework. 
	The TBox signature consists of the atomic concepts  $\supplier$, $\proja$, and $\projb$, denoting the set of suppliers of the company,  suppliers involved in Project~A %$\pa$ 
	and those involved in Project~B, %$\pb$, 
	respectively, and contains the axioms $\proja \ISA \supplier$ and $\projb \ISA \supplier$, stating that each individual instance of $\proja$ or $\projb$ is also instance of $\supplier$. Data protection is specified through the policy $\P=\{\forall x \per \proja(x) \land \projb(x) \rightarrow \bot\}$. 
	The following functions are IB censors for $\T$ and $\pol$:
	\begin{itemize}[noitemsep]
		\item $\optstcensA$: given an ABox $\A$, $\optstcensA(\A)$ returns the set $\cll{\CQL}{\mathit{\A_{\pa}}}$ of BCQs,
		where $\A_{\pa}$ is %the ABox
		obtained from $\A$ by removing the assertion $\proja(c)$, for each individual $c$ such that both $\proja(c)$ and $\projb(c)$ are in $\A$ (note that for every ABox $\A$, $\A$ and $\mathit{\A_{\pa}}$ are indistinguishable w.r.t.\ $\optstcensA$. Similarly in the following censors).
		\item $\optstcensB$: given an ABox $\A$, $\optstcensB(\A)$ returns the set $\cll{\CQL}{\mathit{\A_{\pb}}}$ of BCQs,
		where $\A_{\pb}$ is %the ABox
		obtained from $\A$ by removing the assertion $\projb(c)$, for each individual $c$ such that both $\proja(c)$ and $\projb(c)$ are in $\A$.
		\item $\intcens$:  given an ABox $\A$, $\intcens(\A)$ returns the set $\cll{\CQL}{\mathit{\A_{sup}}}$ of BCQs, where $\A_{sup}$ is obtained from $\A$ by adding the assertion $\supplier(c)$ and removing $\proja(c)$ and $\projb(c)$, for each individual $c$ such that both $\proja(c)$ and $\projb(c)$ are in $\A$. \qed
	\end{itemize}		
\end{example}

It is easy to see that an IB censor always exists, but, as Example~\ref{ex:optimal-censor} shows, there may be many IB censors for a TBox $\T$ and a policy $\pol$, and so it is reasonable to look for censors preserving as much information as possible. Formally, given two IB censors $\cens$ and $\cens'$ for $\T$ and $\pol$, we say that $\cens'$ is \emph{more informative} than $\cens$ if: \emph{(i)} for every ABox $\A$, $\cens(\A) \subseteq\cens'(\A)$, and \emph{(ii)} there exists an ABox $\A'$ such that $\cens(\A') \subset \cens'(\A')$. Optimal censors are then defined as follows.

\begin{definition}\label{def:optima-st-censor}
	Let $\T$ be a DL TBox and $\pol$ be a policy. An IB censor $\cens$ %$\stcens_{\T,\pol}$
	for $\T$ and $\pol$	is \textit{optimal} if there does not exist any other IB censor %$\stcens'_{\T,\pol}$ 
	for $\T$ and $\pol$	that is more informative than $\cens$. %$\stcens_{\T,\pol}$.
	The set of all the optimal IB censors for $\T$ and $\pol$ is denoted with $\stcenset{\T,\P}$.
\end{definition}  

%\begin{definition}\label{def:optima-st-censor}
%	Let $\T$ be a DL TBox, $\pol$ be a policy, and %$\stcens_{\T,\pol}$ 
%	$\cens$ be a strict censor for $\T$ and $\pol$.
%	We say that $\cens$ %$\stcens_{\T,\pol}$ 
%	is an \textit{optimal strict censor} for $\T$ and $\pol$ if there does not exist any other strict censor %$\stcens'_{\T,\pol}$ 
%	$\cens'$ for $\T$ and $\pol$ such that $(i)$ for every ABox $\A$,  %$\stcens_{\T,\pol}(\A) \subseteq \stcens'_{\T,\pol}(\A)$ 
%	$\cens(\A) \subseteq \cens'(\A)$
%	and $(ii)$ there exists an ABox $\A'$ such that %$\stcens_{\T,\pol}(\A') \subset \stcens'_{\T,\pol}(\A')$. 
%	$\cens(\A') \subset \cens'(\A')$.
%	%%
%	The set of all the optimal strict censors for $\T$ and $\pol$ is denoted with $\stcenset{\T,\P}$.
%\end{definition}  

\begin{example}
	Among the censors of  Example~\ref{ex:optimal-censor}, $\intcens \not \in \stcenset{\T,\P}$, since both $\optstcensA$ and $\optstcensB$ are more informative than $\intcens$. It can be then verified that $\optstcensA$ and $\optstcensB$ are the only optimal IB censors for $\T$ and $\pol$. \qed
\end{example}

%\begin{example}
%	We refer to Example~\ref{ex:optimal-censor}, and observe that both $\optstcensA$ and $\optstcensB$ are optimal strict censors for $\T$ and $\pol$, whereas $\intcens \not \in \stcenset{\T,\P}$ since both $\optstcensA$ and $\optstcensB$ preserve more information than $\intcens$. \qed
%\end{example}

%%% Local Variables:
%%% mode: latex
%%% TeX-master: "main"
%%% save-place: t
%%% End:

	%\section{Relation of other notion of CQE}
\section{IB censors vs. CP censors}
\label{sec:CP-censors}

\newcommand{\cpcens}{\mathsf{cp\_cens}}

In~\cite{GKKZ15}, a different notion of censor, named \textit{confidentiality-preserving (CP)} censor, has been proposed. Intuitively, a CP censor establishes which are the BCQs entailed by a TBox and a given ABox  that can be disclosed without violating the policy. We report below the definition given in~\cite{LeRS19}, which generalizes CP censors to any language $\L \subseteq \FOL$, called the censor language.

\begin{definition}[Confidentiality-preserving censor]\label{def:confidentiality-preserving-censor}
	Let $\T$ be a DL TBox, $\pol$ be a policy, and $\lc \subseteq \FOL$ be a language. A \emph{confidentiality-preserving (CP) censor} in $\lc$ for $\T$ and $\P$ is a function $\cens(\cdot)$ %$\lccens{\CQL}_{\T,\pol}
	that, for each ABox $\A$, returns a set $\cens(\A) \subseteq \cll{\L}{\A}$ such that $\T \cup  \pol \cup \cens(\A) $ is a consistent FO theory.
\end{definition}

%Similarly to strict censors, \emph{optimal CP censors} in a certain language $\lc$ are those CP censors in $\lc$ that maximize the set of sentences they return. 
%The definition is specular to Definition~\ref{def:optima-st-censor}.

The notion of more informative censor previously given for IB censors can be naturally extended to CP censors, and we can thus define optimal censors also in this case.

\begin{definition}\label{def:optimal-cp-censor}
	Let $\T$ be a DL TBox, $\pol$ be a policy, and $\lc \subseteq \FOL$ be a language. A CP censor $\cens$ in $\lc$ for $\T$ and $\pol$ is \textit{optimal} if there does not exist any other CP censor in $\lc$ for $\T$ and $\pol$	that is more informative than $\cens$.
	The set of all the optimal CP censors in $\lc$ for $\T$ and $\pol$ is denoted with $\cpcenset{\lc}{\T,\P}$.
\end{definition}

\begin{example}\label{ex:cp-censor}
	Consider $\T$ and $\pol$ as defined in Example~\ref{ex:optimal-censor}. An optimal CP censor $\optcpcensA$ in $\CQL$ for $\T$ and $\pol$ is defined as follows: given an ABox $\A$, $\optcpcensA(\A)$ returns the set of BCQs obtained by removing from $\cll{\CQL}{\A}$ every query containing the atom 
	$\proja(c)$, for each individual $c$ such that  both $\proja(c)$ and $\projb(c)$ are in $\A$.
	
	%OLD example:
	%	Notice that $\optcpcensA$ is not a strict censor for $\T$ and $\pol$. Indeed, consider the ABox $\A=\{\get(c,\meda),\get(c,\medb)\}$. We have that $\optcpcensA(\A)=\{\phi \in \CQL \mid \T \cup \S \models \phi\}$, where $\S=\{\get(c,\meda),\exists x. \get(x,\medb)\}$. It is not hard to see that there exists no ABox $\A'$ such that $\A'$ and $\A$ are indistinguishable \wrt $\lccens{\CQL}_{\T,\pol}^{4}$ and $\T \cup \pol \cup \A'$ is consistent. 
	We soon notice that $\optcpcensA$ is instead not an IB censor. Indeed, consider the ABox $\A=\{\proja(c),\projb(c)\}$. We have that $\optcpcensA(\A)=\{\phi \mid \phi \in \CQL \textrm{ and } \T \cup \S \models \phi\}$, where $\S=\{\exists x. \proja(x),\projb(c)\}$. It is not hard to see that there exists no ABox $\A'$ such that $\A'$ and $\A$ are indistinguishable \wrt $\optcpcensA$ and $\T \cup \pol \cup \A'$ is consistent.   \qed
\end{example}

Let $\A$ be an ABox and $\cens$ be either an IB or a CP censor, the set $\cens(\A)$ is called \emph{theory of the censor} $\cens$ for $\A$.

The following theorem explains the relation between IB censors and CP censors.  

\begin{theorem}\label{lem:strict-conf}
	Let $\T$ be a DL TBox and $\pol$ be a policy. If $\cens$ is an IB censor for $\T$ and $\pol$, then it is a CP censor in $\CQL$ for $\T$ and $\pol$. The converse does not necessarily hold.
\end{theorem}
\begin{proof}
	Let $\cens$ be an IB censor for $\T$ and $\pol$. Consider an arbitrary ABox $\A$. According to Definition~\ref{def:strict-censor}, there exists an ABox $\A'$ such that $\cens(\A)=\cens(\A')$ and $\T \cup \pol \cup \A'$ is consistent. Since by definition $\cens(\A')$ contains only sentences $\phi \in \CQL$ logically implied by $\T \cup \A'$ (i.e., BCQs $\phi$ such that $\phi \in \cll{\CQL}{\A'}$) and $\T \cup \pol \cup \A'$ is consistent, we have that $\T \cup \pol \cup \cens(\A')$ is consistent as well. Due to the equivalence $\cens(\A')=\cens(\A)$, we derive that $\T \cup \pol \cup \cens(\A)$ is consistent. To conclude the implication part observe that, by definition, $\cens(\A) \subseteq \cll{\CQL}{\A}$.

	As for the converse, %the last paragraph of 
	Example~\ref{ex:cp-censor} shows that the CP censor $\optcpcensA$ in $\CQL$ for $\T$ and $\P$ is not an IB censor for $\T$ and $\pol$.
\end{proof}

We also notice that optimal IB censors are \emph{not} necessarily optimal CP censors in $\CQL$.
Indeed, consider Examples~\ref{ex:optimal-censor} and~\ref{ex:cp-censor}. We have that %\myi
$\optstcensA \in \stcenset{\T,\P}$ but, even if, as shown by Theorem~\ref{lem:strict-conf}, it is a CP censor in $\CQL$ for $\T$ and $\pol$, %we have that 
$\optstcensA \not \in \cpcenset{\CQL}{\T,\P}$ (%indeed, 
it is easy to see that $\optcpcensA$ is more informative than $\optstcensA$). We also know from Example~\ref{ex:cp-censor} that the optimal CP censor $\optcpcensA$ in $\CQL$ for $\T$ and $\P$ is not an IB censor, and thus $\optcpcensA \not \in \stcenset{\T,\P}$.
%; and \myii $\optcpcensA \in  \cpcenset{\CQL}{\T,\P}$ but $\optcpcensA \not \in \stcenset{\T,\P}$ (in fact, $\optcpcensA$ is not strict for $\T$ and $\pol$ according to Definition~\ref{def:strict-censor}). 
However, if an optimal CP censor in $\CQL$ for $\T$ and $\pol$ is also an IB censor %a strict censor for $\T$ and $\pol$, 
then it is an optimal IB censor for $\T$ and $\pol$, as stated below.

\begin{corollary}\label{lem:ocpc-osc}
	Let $\T$ be a DL TBox and $\pol$ be a policy. 
	If $\cens \in \cpcenset{\CQL}{\T,\P}$ and $\cens$ is an IB censor for $\T$ and $\pol$, then $\cens \in \stcenset{\T,\P}$. The converse does not necessarily hold.
\end{corollary}
\begin{proof}
Theorem~\ref{lem:strict-conf} implies that the set $\I\B$ of IB censors for $\T$ and $\pol$ is a subset of the set $\C\P$ of CP censors in $\CQL$ for $\T$ and $\pol$. Thus, since for a censor $\cens \in \cpcenset{\CQL}{\T,\P}$ there does not exists in $\C\P$ a censor $\cens'$ that is more informative than $\cens$, such $\cens'$ cannot exists in $\I\B$ too.

As a counterexample for the converse, as said above, $\optstcensA$ is in $\stcenset{\T,\P}$ but not in $\cpcenset{\CQL}{\T,\P}$. 
\end{proof}

Actually, the relation between the two optimality notions of censor depends on the 
censor language adopted for the CP censors. 
In particular, for $\GA$, the set of the theories of the optimal IB censors for a TBox $\T$ and a policy $\P$ coincides with the set of the deductive closures 
$\cll{\CQL}{\cdot}$ of the theories of the optimal CP censors in $\GA$ for $\T$ and $\P$. This property is formalized by the following theorem, which is crucial to establish the complexity results of the next section.

\begin{theorem}\label{thm:cq-st2ga-cp}
	Let $\T$ be a DL TBox and $\pol$ be a policy. Then, $\stcens \in \stcenset{\T,\P}$ iff there exists a CP censor $\cpcens \in \cpcenset{\GA}{\T,\P}$ such that, for each ABox $\A$, $\cll{\CQL}{\cpcens(\A)}=\stcens(\A)$.
\end{theorem}
\begin{proof}
	($\Leftarrow$). Suppose that there exists a CP censor $\cpcens \in \cpcenset{\GA}{\T,\P}$ such that $\cll{\CQL}{\cpcens(\A)} = \stcens(\A)$ for each ABox $\A$. Observe that, since $\cpcens$ is an optimal censor in $\GA$ for $\T$ and $\pol$, we have that \myi $\cpcens(\A)=\cll{\GA}{\A}$ for each ABox $\A$ such that $\T \cup \pol \cup \A$ is consistent (otherwise, we easily get a contradiction on the optimality of $\cpcens$), and \myii $\T \cup \P \cup \cpcens(\A)$ is consistent for each ABox $\A$, where $\cpcens(\A)$ can be seen as another ABox. From the above considerations, and the fact that $\stcens(\A)=\cll{\CQL}{\cpcens(\A)}$ holds by assumption for each ABox $\A$, we have that, for each ABox $\A$, also the following hold: \myi $\stcens(\cpcens(\A))=\stcens(\A)$ (i.e., $\A$ and $\cpcens(\A)$ are indistinguishable \wrt $\stcens$), and \myii $\T \cup \pol \cup \stcens(\A)$ is consistent because $\T \cup \P \cup \cpcens(\A)$ is consistent. This, together with the fact that $\stcens(\A) \subseteq \cll{\CQL}{\A}$ for each ABox $\A$ (since $\stcens(\A)=\cll{\CQL}{\cpcens(\A)}$ and $\cpcens(\A) \subseteq \cll{\GA}{\A}$), implies that $\stcens$ is an IB censor for $\T$ and $\pol$.
	
	We now prove its optimality by way of contradiction. Suppose, for the sake of contradiction, that $\stcens$ is not an optimal IB censor for $\T$ and $\pol$, i.e., there exists an IB censor $\stcens'$ for $\T$ and $\pol$ such that $\stcens(\A) \subseteq \stcens'(\A)$ for each ABox $\A$, and there exists an ABox $\A'$ such that $\stcens(\A') \subset \stcens'(\A')$. Since $\stcens'$ is an IB censor for $\T$ and $\pol$, there is an ABox $\A'_i$ such that $\A'_i$ and $\A'$ are indistinguishable \wrt $\stcens'$ (i.e., $\stcens'(\A')=\stcens'(\A'_i)$) and $\T \cup \pol \cup \A'_i$ is consistent. Since by definition $\stcens'(\A'_i) \subseteq \cll{\CQL}{\A'_i}$, the following inclusions hold:
	$$
	\stcens(\A') \subset \stcens'(\A')=\stcens'(\A'_i) \subseteq \cll{\CQL}{\A'_i}.
	$$
	By assumption, moreover, we know that $\cll{\CQL}{\cpcens(\A')} = \stcens(\A')$, and therefore $\cll{\CQL}{\cpcens(\A')} \subset \cll{\CQL}{\A'_i}$. It follows that $\cpcens(\A') \subset \A'_i$, i.e., there is a ground atom $\psi$ such that $\psi \in \A'_i$ and $\psi \not \in \cpcens(\A')$. But then, consider the function $\cpcens'$ with $\cpcens'(\A) = \cpcens(\A)$ for each ABox $\A$ such that $\A \neq \A'$ and $\cpcens'(\A') = \cll{\GA}{\A'_i}$. Clearly, due to the facts that $\cpcens$ is a CP censor in $\GA$ for $\T$ and $\pol$ and $\T \cup \pol \cup \cll{\GA}{\A'_i}$ is consistent (because $\T \cup \pol \cup \A'_i$ is consistent), we have that $\cpcens'$ is a CP censor in $\GA$ for $\T$ and $\pol$ as well. Observe, however, that $\cpcens(\A) \subseteq \cpcens(\A)$ for each ABox $\A$, and $\cpcens(\A') \subset \cpcens'(\A')=\cll{\GA}{\A'_i}$. In particular, the ground atom $\psi$ is such that $\psi \in \A'_i$ (and thus $\psi \in \cpcens'(\A') = \cll{\GA}{\A'_i}$) and $\psi \not \in \cpcens(\A')$. Therefore $\cpcens'$ is a CP censor in $\GA$ for $\T$ and $\pol$ that is more informative than $\cpcens$, and this contradicts the assumption that $\cpcens \in \cpcenset{\GA}{\T,\P}$, as required.

	($\Rightarrow$) In the proof, we will make use of the following claim.
	
	\begin{claim}
		Let $\T$ be a DL TBox, $\pol$ be a policy, and $\stcens$ be an IB censor for $\T$ and $\pol$. If $\stcens \in \stcenset{\T,\P}$, then the following hold:
		\begin{itemize}
			\item[\myi] $\stcens(\A)=\cll{\CQL}{\A}$ for each ABox $\A$ such that $\T \cup \pol \cup \A$ is consistent.
			\item[\myii] $\stcens(\A)=\cll{\CQL}{\A'}=\stcens(\A')$ for each ABox $\A$, where $\A'$ is the ABox such that $\A$ and $\A'$ are indistinguishable \wrt $\stcens$ and $\T \cup \pol \cup \A'$ is consistent (such an ABox $\A'$ is guaranteed to exists due to the fact that $\stcens$ is an IB censor for $\T$ and $\pol$).
		\end{itemize}
	\end{claim}

	\begin{proof}
		Assume that $\stcens \in \stcenset{\T,\P}$.
		
		Suppose, for the sake of contradiction, that \myi does not hold, i.e., there exists an ABox $\A_c$ such that $\stcens(\A_c) \subset \cll{\CQL}{\A_c}$ and $\T \cup \pol \cup \A_c$ is consistent. But then, consider the function $\stcens'$ with $\stcens'(\A)=\cll{\CQL}{\A_c}$ for each ABox $\A$ such that $\A$ and $\A_c$ are indistinguishable \wrt $\stcens$ (obviously, $\stcens'(\A_c)=\cll{\CQL}{\A_c}$ since indistinguishability \wrt an IB censor for a DL TBox $\T$ and policy $\pol$ always forms an equivalence relation), and $\stcens'(\A)=\stcens(\A)$ for each ABox $\A$ such that $\A$ and $\A_c$ are not indistinguishable \wrt $\stcens$. Observe that, for each pair of ABoxes $\A_1$ and $\A_2$, we have that $\A_1$ and $\A_2$ are indistinguishable \wrt $\stcens$ if and only if they are indistinguishable \wrt $\stcens'$. Furthermore, since $\stcens$ is an IB censor for $\T$ and $\pol$, and since $\T \cup \pol \cup \A_c$ is consistent (and therefore also $\T \cup \pol \cup \stcens'(\A)=\cll{\CQL}{\A_c}$ is consistent for each ABox $\A$ such that $\A$ and $\A_c$ are indistinguishable \wrt $\stcens$), it can be easily verified that $\stcens'$ is an IB censor for $\T$ and $\pol$ that more informative than $\stcens$ (in particular, $\stcens(\A) \subset \stcens'(\A)=\cll{\CQL}{\A}$ for each ABox $\A$ such that $\A$ and $\A_c$ are indistinguishable \wrt $\stcens$), thus contradicting the assumption that $\stcens$ is an optimal IB censor for $\T$ and $\pol$, as required.

		As for \myii, let $\A$ be an arbitrary ABox. Consider the ABox $\A'$ such that $\A$ and $\A'$ are indistinguishable \wrt $\stcens$ (i.e., $\stcens(\A)=\stcens(\A')$), and $\T \cup \pol \cup \A'$ is consistent. From \myi, we derive that $\stcens(\A')=\cll{\CQL}{\A'}$, and therefore $\stcens(\A)=\stcens(\A')=\cll{\CQL}{\A'}$.
	\end{proof} 
	
	Suppose that $\stcens \in \stcenset{\T,\P}$.
	Consider the function $\cpcens$ with $\cpcens(\A)=\GA \cap \stcens(\A)$ for each ABox $\A$. In other words, for each ABox $\A$, $\cpcens(\A)$ returns the set of all and only the ground atoms occurring in $\stcens(\A)$. From the definition of $\cpcens$ and from Claim \myi, it is easy to see that $\cll{\CQL}{\cpcens(\A)} = \stcens(\A)$ for each ABox $\A$. We now prove that $\cpcens \in \cpcenset{\GA}{\T,\P}$. 
	
	Observe that, by the assumption that $\stcens$ is an IB censor for $\T$ and $\pol$, we have that $\stcens(\A) \subseteq \cll{\CQL}{\A}$ (and therefore $\cpcens(\A) \subseteq \cll{\GA}{\A}$) for each ABox $\A$. Furthermore, for each ABox $\A$, $\T \cup \pol \cup \A'$ is consistent (implying that $\T \cup \pol \cup \stcens(\A)=\stcens(\A')$ is consistent), where $\A'$ is the ABox such that $\A$ and $\A'$ are indistinguishable \wrt $\stcens$, and therefore, since $\cpcens(\A) = \GA \cap \stcens(\A)$ for each ABox $\A$, we derive that $\T \cup \pol \cup \cpcens(\A)$ is consistent for each ABox $\A$. Thus, $\cpcens$ is a CP censor in $\GA$ for $\T$ and $\pol$. 
	
	We now prove its optimality by contradiction. Suppose, for the sake of contradiction, that $\cpcens$ is not an optimal CP censor in $\GA$ for $\T$ and $\pol$, i.e., there exists an optimal CP censor $\cpcens'$ in $\GA$ for $\T$ and $\pol$ such that $\cpcens(\A) \subseteq \cpcens'(\A)$ for each ABox $\A$, and there exists an ABox $\A'$ such that $\cpcens(\A') \subset \cpcens'(\A')$ (observe that, by definition, an optimal CP censor in $\GA$ for $\T$ and $\pol$ always exists). Consider now the function $\stcens'$ with $\stcens'(\A)=\cll{\CQL}{\cpcens'(\A)}$ for each ABox $\A$. Since $\cpcens' \in \cpcenset{\GA}{\T,\P}$ and $\cll{\CQL}{\cpcens'(\A)} = \stcens'(\A)$ for each ABox $\A$, using the (\textit{$\Leftarrow$}) shown before, we derive that $\stcens' \in \stcenset{\T,\P}$. Observe that: \myi for each ABox $\A$, we have that $\cpcens(\A) \subseteq \cpcens'(\A)$, $\stcens(\A)=\cll{\CQL}{\cpcens(\A)}$, and $\stcens'(\A)=\cll{\CQL}{\cpcens'(\A)}$; \myii there exists an ABox $\A'$ such that $\cpcens(\A') \subset \cpcens'(\A')$, i.e., there is a ground atom $\psi$ such that $\psi \in \cpcens'(\A')$ and $\psi \not \in \cpcens(\A')$. From \myi, however, we easily derive that $\stcens(\A) \subseteq \stcens'(\A)$ for each ABox $\A$. Furthermore, since for the ABox $\A'$ $\cpcens(\A') \subset \cpcens'(\A')$, and since by definition $\cpcens(\A)=\GA \cap \stcens(\A)$ for each ABox $\A$, we have that $\GA \cap \stcens(\A') \subset \cpcens'(\A')$. Due to the fact that $\stcens(\A)=\cll{\CQL}{\cpcens(\A)}$ for each ABox $\A$, we derive $\GA \cap \cll{\CQL}{\cpcens(\A')} \subset \cpcens'(\A')$. It is not hard to see that this latter fact implies that $\cll{\CQL}{\GA \cap \cll{\CQL}{\cpcens(\A')}} \subset \cll{\CQL}{\cpcens'(\A')}$. In particular, the ground atom $\psi$ is such that $\psi \in \cll{\CQL}{\cpcens'(\A')}$ and $\psi \not \in \cll{\CQL}{\GA \cap \cll{\CQL}{\cpcens(\A')}}$.
	
	Thus, since as shown in the previous steps the following equalities hold
	\begin{align*}
		&\cll{\CQL}{\GA \cap \cll{\CQL}{\cpcens(\A')}}=\\
		&=\cll{\CQL}{\GA \cap \stcens(\A')}=\cll{\CQL}{\cpcens(\A')}=\\
		&=\stcens(\A'),
	\end{align*}
	and since $\cll{\CQL}{\cpcens'(\A')}=\stcens'(\A')$, we derive that $\stcens(\A') \subset \stcens'(\A')$. Therefore $\stcens'$ is an IB censor for $\T$ and $\pol$ more informative than $\stcens$, and this contradicts the assumption that $\stcens \in \stcenset{\T,\P}$, as required.
\end{proof}

	\section{Query answering under optimal IB censors}\label{sec:cqe-over-strict}

\newcommand{\computeIB}{\textsf{OptGACensor}\xspace}
\newcommand{\IBEntailment}{\textsf{IB-Entailment}\xspace}
\newcommand{\strictEntailment}{\IBentailment}
\newcommand{\Thc}{\mathsf{\textit{Th}}\xspace}

In this section we study query answering under IB censors over $\dlliter$ ontologies. In particular, we consider entailment of BCQs specified over the signature of the ontology.

A possible strategy for addressing this problem is to choose only one IB censor among the optimal ones, and use it to alter the answers to user queries. 
In the absence of a criterion for determining which censor is the best for our purposes, the choice of the optimal censor %can be
is made in an arbitrary way (like in~\cite{BiskupBonatti07,GKKZ13}).
Towards the realization of an optimal IB censor, we first provide the algorithm $\computeIB$ (Algorithm~\ref{alg:computeOptimalGACensor}),  which implements a function that, for every $\dlliter$ TBox $\T$ and every policy $\pol$, corresponds to an optimal CP censor in $\GA$ for $\T$ and $\pol$.
%%(for every ABox $\A$ the algorithm computes the theory of such optimal censor for $\A$).
Then we explain how to use $\computeIB$ to establish BCQs entailment under an optimal IB censor by exploiting Theorem~\ref{thm:cq-st2ga-cp}. 
%Given a $\dlliter$ TBox, a policy $\pol$, and an ABox $\A$, $\computeIB$ returns an ABox that indeed coincides with the theory $\cens(\A)$ of an optimal CP censor in $\GA$ for $\T$ and $\P$. 
The algorithm first computes the set $\A_\T$ of ground atoms entailed by $\T \cup \A$. Then, it iteratively picks a ground atom $\alpha$ from $\A_\T$ following the lexicographic order, and adds $\alpha$ to the ABox $\Thc$ if $\T \cup \Thc \cup \alpha$ does not violate the policy $\P$. 
%It is easy to see that the algorithm runs in \PTIME in the size of $\A$, as stated by the following theorem, which also establishes the correctness of the algorithm.
The following theorem establishes the correctness and complexity of the algorithm.

\begin{algorithm}[tb]
	\begin{tabbing}
		\textbf{input:} \= a $\dlliter$ TBox $\T$, 
		 a policy $\pol$, an ABox $\A$; \\
		\textbf{output:} an ABox; \\
%	\textbf{output:} a theory of a CP censor in $\GA$ for $\T$ and $\pol$ for $\A$;\\		
     1)~ \= $\A_{\T} \leftarrow \cll{\GA}{\A}$;\\ 
		2) \> $\Thc \leftarrow \emptyset$;\\
		3) \> \textbf{while} $\A_{\T}$ is not empty \textbf{do}: \\
		4) \> ~~~\= \textbf{let} $\alpha$ be the  lexicographically first assertion in $\A_{\T}$;\\
		5)	\>	\> $\A_{\T} \leftarrow \A_{\T} \setminus \{ \alpha \}$;\\
		6)	\>	\> \textbf{if} $\T \cup \Thc \cup \{ \alpha \} \cup \pol$ is consistent \textbf{then}\\
		7)	\>	\>~~~\= $\Thc \leftarrow \Thc \cup \{ \alpha \}$;    \\
		8) \> \textbf{return} $\Thc$;
	\end{tabbing}
\vspace{-0.2cm}
\caption{\computeIB}\label{alg:computeOptimalGACensor}
\end{algorithm}

\begin{theorem}\label{thm:onecensor}	
	Let $\T$ be a $\dlliter$ TBox and $\pol$ be a policy.	
	 There exists a censor $\cens \in \cpcenset{\GA}{\T,\P}$ such that, for each ABox $\A$, $\computeIB(\T,\pol,\A)$ $(i)$ returns $\cens(\A)$ and $(ii)$ runs in polynomial time  in the size of $\A$. 
%	 $(i)$ $\computeIB(\T,\pol,\A) =   \cens(\A)$, and $(ii)$ $\computeIB(\T,\pol,\A)$ runs in \PTIME in the size of $\A$. 
\end{theorem}
\begin{proof} 
	For each ABox $\A$, the set $\Thc$ returned by the algorithm contains only assertions in $\cll{\GA}{\A}$,  that is, it contains only assertions in $\GA$ entailed by $\T\cup\A$. Moreover, step~6 of the algorithm checks that $\Thc$ is consistent with $\T$ and $\pol$. Hence, according to Definition~\ref{def:confidentiality-preserving-censor}, the algorithm implements a CP censor $\cpcens$ in $\GA$ for $\T$ and $\pol$. It is also immediate to verify that $\cpcens$  is optimal. Indeed, suppose, by way of contradiction, that there exists an ABox $\A$ and a censor $\cpcens'$ such that $\cpcens(\A) \subset\cpcens'(\A)$ and $\cpcens(\A') \subseteq\cpcens'(\A')$ for every other ABox $\A'$. This means that there exists an assertion $\alpha \in \cll{\GA}{\A}$ such that $\alpha$ in $\cpcens'(\A) \setminus \cpcens(\A)$, but since $\alpha$ is not in $\cpcens(\A)$ then $\T \cup \cpcens(\A) \cup \{ \alpha \} \cup \pol$ has to be inconsistent (step~6 of the algorithm), and so $\T \cup \cpcens'(\A) \cup \pol$ is inconsistent too, which contradicts the fact that  $\cpcens'$ is a CP censor. 
%	Finally, since we have shown that the algorithm implements an optimal CP censors in $\GA$ for $\T$ and $\pol$, the thesis immediately follows from Theorem~\ref{thm:cq-st2ga-cp}.    

As for the complexity, note that the algorithm iterates on the set of ABox assertions $\cll{\GA}{\A}$ by choosing an assertion $\alpha$ and, in each iteration, it checks if $\T \cup \Thc \cup \{ \alpha\} \cup \P$ is consistent. Clearly, the algorithm terminates since $\cll{\GA}{\A}$ is finite. Moreover, the thesis follows  from the fact that given a $\dlliter$ TBox, a policy $\pol$ (i.e. a set of denial assertions), and an ABox $\Thc \cup \{ \alpha\}$, checking if $\T \cup \Thc \cup \{ \alpha\} \cup \P$ is consistent can be done in $\aczero$ w.r.t. to the size of $\Thc \cup \{ \alpha\} $~\cite{LLRRS15}, that the set $\cll{\GA}{\A}$ can be computed in polynomial time w.r.t. $|\A|$ and that its size is polynomial w.r.t. to $|\A|$ as well. 
\end{proof}

%Algorithm~\ref{alg:computeOptimalGACensor} shows that given a $\dlliter$ TBox $\T$, a policy $\P$ and an ABox $\A$, it is possible to compute an ontology $\T \cup \computeIB(\T,\pol,\A)$ that is logically equivalent to an optimal IB censor for $\T$ and $\P$. Moreover, the algorithm runs in polynomial time with respect to the size of the ABox.
%%
%From Theorem~\ref{thm:onecensor}, it immediately follows that if $\T$ is in $\dlliter$, deciding if $\T \cup \cens(\A) \models q$, where $q$ is a BCQ and $\cens$ is the optimal IB censor implemented by Algorithm~\ref{alg:computeOptimalGACensor}, can be done in PTIME in data complexity. 

From Theorem~\ref{thm:cq-st2ga-cp} and Theorem~\ref{thm:onecensor} it follows that, to establish if a BCQ $q$ is entailed by $\T \cup \A$ under an optimal IB censor for $\T$ and $\P$, it is sufficient to verify whether $\T \cup \computeIB(\T,\pol,\A) \models q$, which can be done in polynomial time in the size of $\A$.

We note also that it is possible to implement different optimal IB censors (actually, every optimal IB censor) by modifying the order in which the ABox assertions from the set $\cll{\GA}{\A}$ are selected by the algorithm.
%%
%% Actually, it can be shown that, given an $\dlliter$ TBox and a policy $\pol$, one can implement all the possible optimal censors in $\stcenset{\T,\P}$  by using a maximum of $|\cll{\GA}{\A}|!$ algorithms, and that each of them runs in polynomial time with respect to the size of the ABox (as long as the strategy for ordering the assertions in $\cll{\GA}{\A}$ can be implemented in polynomial time with respect to the size of the ABox as well).     

Depending on the application at hand, the approach of randomly choosing a censor may not always be considered appropriate~\cite{GKKZ13}. For this reason, in~\cite{LeRS19} the authors suggest to use a form of skeptical entailment over (the theories of) all the optimal censors, i.e., they propose a CQE framework in which a query has a positive answer if it is entailed by each optimal censor. %coupled with the TBox. 
In the same spirit, we define the following decision problem. 

\begin{definition}
	Let $\T$ be a DL TBox, $\pol$ be a policy, $\A$ be an ABox, and $q$ be a BCQ. $\IBEntailment(\T,\P,\A,q)$ is the problem of deciding whether $q \in  \cens(\A)$ for every  $\cens$ in $\stcenset{\T,\P}$.  
\end{definition}

By exploiting Theorem~\ref{thm:cq-st2ga-cp} and the results given in~\cite{LeRS19}, we can provide the following theorem. 

\begin{theorem} 	
	\label{thm-complexity-IB-entailment}
	Let $\T$ be a $\dlliter$ TBox, $\pol$ be a policy, $\A$ be an ABox, and $q$ be a BCQ. %The problem 
	Then, $\IBEntailment(\T,\P,\A,q)$ is coNP-complete in data complexity.
\end{theorem}
\begin{proof}
  The result immediately follows from Theorem~\ref{thm:cq-st2ga-cp} and from~\cite[Theorem 6]{LeRS19}, which states that deciding if $\T \cup \cpcens(\A) \models q$, for every $\cpcens$ in $\GA$-$\censet{\T,\P}$ is coNP-complete in data complexity.
\end{proof}

	\newcommand{\iarEntailment}{\textsf{IAR-Entailment}\xspace}
\newcommand{\qibcens}{\mathsf{qib\_cens}\xspace}
\newcommand{\atomrewr}{\textsf{atomRewr}\xspace}
\newcommand{\intEntailment}{\textsf{QIB-Entailment}\xspace}
\newcommand{\q}{q\xspace}
\newcommand{\Qs}{\Q_{sec}\xspace}
\newcommand{\secrets}{\mathit{secrets}}
\renewcommand{\unify}{\textsf{saturate}}
\newcommand{\satq}{\textsf{uncoverQueries}}
\newcommand{\minsatq}{\textsf{minUncoverQueries}\xspace}

\section{Approximating optimal IB censors}
\label{sec:cqe-approximation}

As stated in Theorem~\ref{thm-complexity-IB-entailment}, $\IBEntailment$ is in general intractable in data complexity. Towards a practical approach to CQE, in this section we consider a different entailment problem that approximates $\IBEntailment$, and we show that 
its data complexity is in $\aczero$ (i.e., the same complexity of evaluating FO queries over a database).
%
%For data-intensive applications, the complexity of CQE under the entailment approaches that we studied in the previous section can be an obstacle in practical use. For this reason, in this section, we propose a new entailment approach that approximates the ones given before, under which entailment of BCQs is in $\aczero$ in data complexity. 
%
The approximation we propose consists in considering a non-necessarily optimal IB censor whose theory, for every ABox, is as close as possible to the theories of all the optimal IB censors.

%\nb{Questa def.\ potrebbe anche comparire prima}Given two censors $\cens,\cens'$ for $\T$ and $\pol$, we say that $\cens'$ is \emph{more informative} than $\cens$ if: \emph{(i)} for every ABox $\A$, $\cens(\A) \subseteq\cens'(\A)$, and \emph{(ii)} there exists an ABox $\A'$ such that $\cens(\A') \subset \cens'(\A')$.

\begin{definition}[AIB censor and QIB censor]\label{def:qos-censor}
%  Given an IB censor $\stcens$ for $\T$ and $\pol$:
%  \begin{enumerate}
%  \item[$(i)$] we say that $\stcens$ is an \emph{approximation of the optimal IB censors (AIB censor}) for $\T$ and $\pol$ if, for every $\stcens' \in \stcenset{\T,\P}$ and for every ABox $\A$, $\stcens(\A) \subseteq \stcens'(\A)$;
%    \item[$(ii)$] we say that $\stcens$ is a \textit{quasi-optimal IB censor (QIB censor)} for $\T$ and $\pol$ if $\stcens$ is an AIB censor for $\T$ and $\pol$ and there exists no AIB censor $\stcens'$ for $\T$ and $\pol$ that is more informative than $\stcens$.
%\end{enumerate}
Let $\T$ be a DL TBox, let $\pol$ be a policy, and let  $\cens$ be an IB censor for $\T$ and $\pol$. We say that: 
\begin{enumerate}[noitemsep]
	\item[$(i)$] $\cens$ is an \emph{approximation of the optimal IB censors (AIB censor)} for $\T$ and $\pol$ if, for every $\cens' \in \stcenset{\T,\P}$ and for every ABox $\A$, $\cens(\A) \subseteq \cens'(\A)$;
	\item[$(ii)$] $\cens$ is a \textit{quasi-optimal IB censor (QIB censor)} for $\T$ and $\pol$ if $\cens$ is an AIB censor for $\T$ and $\pol$ and there exists no AIB censor $\cens'$ for $\T$ and $\pol$ that is more informative than $\cens$.
\end{enumerate}
\end{definition}

%% \begin{definition}[Quasi-optimal IB censor]\label{def:qos-censor}
%% 	Let $\T$ be a DL TBox, $\pol$ a policy and $\stcens$ an IB censor for $\T$ and $\pol$. 
%% 	We say that $\stcens$ is a \textit{quasi-optimal IB censor (QIB censor)} for $\T$ and $\pol$, if:
%% 	%\begin{itemize}
%% 	%\item[$(i)$] 
%% 	$(i)$ for every $\stcens'_{\T,\pol} \in \stcenset{\T,\P}$ and for every ABox $\A$, $\stcens_{\T,\pol}(\A) \subseteq \stcens'_{\T,\pol}(\A)$;
%% 	%\item[$(ii)$] 
%% 	$(ii)$ there does not exist an IB censor $\stcens''_{\T,\pol}$ such that \emph{(a)} for every $\stcens'_{\T,\pol} \in \stcenset{\T,\P}$ and for every ABox $\A$, $\stcens_{\T,\pol}(\A) \subseteq\stcens''_{\T,\pol}(\A)\subseteq\stcens'_{\T,\pol}(\A)$, and \emph{(b)} there exists an ABox $\A'$ such that $\stcens'_{\T,\pol}(\A) \subset \stcens''_{\T,\pol}(\A)$.
%% 	%\end{itemize} 
%% \end{definition}

\begin{example}
	The IB censor $\intcens$ of Example~\ref{ex:optimal-censor} is a QIB censor for $\T$ and $\pol$ (but $\intcens \not \in \stcenset{\T,\P}$). %(but it is not optimal).  
	\qed
\end{example}

%Given a TBox $\T$ and a policy $\pol$ for $\T$, it is easy to verify that for every pair of QIB censors $\stcens_{\T,\pol}$ and $\stcens'_{\T,\pol}$ for $\T$ and $\pol$, and for every ABox $\A$ for $\T$, we have that $\stcens_{\T,\pol}(\A) = \stcens'_{\T,\pol}(\A)$.   

\newcommand{\lbcens}{\cens_\mathsf{0}}

For QIB censors the following notable property hold.

\begin{theorem}
  Let $\T$ be a DL TBox and let $\pol$ be a policy. A QIB censor for $\T$ and $\pol$ always exists and it is unique.
\end{theorem}
\begin{proof}
  First, observe that the ``least informative'' censor $\lbcens$ such that $\lbcens(\A)=\emptyset$ for every ABox $\A$, satisfies condition $(i)$ of Definition \ref{def:qos-censor}. So, either $\lbcens$ is a QIB censor (i.e., it satisfies condition $(ii)$ of Definition \ref{def:qos-censor}), or there exists a more informative AIB censor (satisfying such condition $(ii)$). This implies the existence of a QIB censor (we recall that censors can return infinite sets of BCQs).

  Then, let us assume that there exists two distinct QIB censors $\cens_\mathsf{a},\cens_\mathsf{b}$ for $\T$ and $\pol$. Then, there exists an ABox $\A$ such that $\cens_\mathsf{a}(\A)\neq\cens_\mathsf{b}(\A)$. Since $\cens_\mathsf{a},\cens_\mathsf{b}$ are IB censors for $\T$ and $\pol$, let $\A_1$ be an ABox such that $\cens_\mathsf{a}(\A_1)=\cens_\mathsf{a}(\A)$ and $\T\cup\pol\cup\A_1$ is consistent, and let $\A_2$ be an ABox such that $\cens_\mathsf{b}(\A_2)=\cens_\mathsf{b}(\A)$ and $\T\cup\pol\cup\A_2$ is consistent.
  Since $\cens_\mathsf{a},\cens_\mathsf{b}$ are AIB censors for $\T$ and $\pol$, for every $\cens'\in \censet{\T,\P}$, $\A_1 \subseteq \cens'(\A)$ and $\A_2 \subseteq \cens'(\A)$, and therefore, $\cll{\CQL}{\A_1\cup\A_2}\subseteq\cens'(\A)$.
  Now observe that $\cens_\mathsf{a}(\A)\subseteq\cll{\CQL}{\A_1}$, $\cens_\mathsf{b}(\A)\subseteq\cll{\CQL}{\A_2}$ and $\cens_\mathsf{a}(\A)\neq\cens_\mathsf{b}(\A)$. This implies that either $\cens_\mathsf{a}(\A)\subset\cll{\CQL}{\A_1\cup\A_2}$ or $\cens_\mathsf{b}(\A)\subset\cll{\CQL}{\A_1\cup\A_2}$. Let us assume, w.l.o.g., that $\cens_\mathsf{a}(\A)\subset\cll{\CQL}{\A_1\cup\A_2}$. Then, let $\cens_\mathsf{a}'$ be the censor such that $\cens_\mathsf{a}'(\A)=\cll{\CQL}{\A_1\cup\A_2}$ and $\cens_\mathsf{a}'(\A')=\cens_\mathsf{a}(\A')$ for every other ABox $\A'$ different from $\A$. Now, $\cens_\mathsf{a}'$ is an AIB censor for $\T$ and $\pol$, since $\cens_\mathsf{a}$ is an AIB censor for $\T$ and $\pol$ and, as shown above, $\cll{\CQL}{\A_1\cup\A_2}$ is a subset of the theory of every optimal IB censor for $\T$ and $\pol$ over $\A$ , and for the same reason $\cll{\CQL}{\A_1\cup\A_2}$ is consistent with $\T\cup\pol$. Moreover, $\cens_\mathsf{a}'$ is more informative that $\cens_\mathsf{a}$, since $\cens_\mathsf{a}(\A)\subset\cens_\mathsf{a}'(\A)$ (and $\cens_\mathsf{a}(\A')=\cens_\mathsf{a}'(\A')$ for every other ABox $\A'$). Consequently, $\cens_\mathsf{a}$ is not a QIB censor for $\T$ and $\pol$, contradicting the hypothesis. This proves that the QIB censor for $\T$ and $\pol$ is unique.   
\end{proof}

Hereinafter, we denote with $\qibcens_{\T,\pol}$ the QIB censor for $\T$ and $\pol$.
%
%With the notion of QIB censor in place, we define the following decision problem. 
%
Entailment of BCQs over QIB censors is then naturally defined as follows.

\begin{definition}
	Let $\T$ be a DL TBox, let $\pol$ be a policy, let $\A$ be an ABox, and let $q$ be a BCQ. $\intEntailment(\T,\P,\A,q)$ is the problem of deciding whether  $q \in \qibcens_{\T,\pol}(\A)$.
\end{definition}

%We now prove that, when $\T$ is a $\dlliter$ TBox, $\intEntailment(\T,\P,\A,q)$ is in $\aczero$ in the size of $\A$. We obtain this result by showing that this problem is FOL-rewritable, that is, we reduce it to the evaluation of a FOL-query over the ABox. 
%
We now focus on the case of $\dlliter$ TBoxes and prove that, in this case, entailment of BCQs under QIB censors is FO-rewritable.
Formally, we say that QIB-entailment in a DL $\L$ is FO-rewritable, if for every TBox $\T$ expressed in $\L$, every policy $\P$ and every BCQ $q$, one can effectively compute an FO query $q_r$ such that for every ABox $\A$, $\intEntailment(\T,\P,\A,q)$ is true iff $\A \models q_r$. We call $q_r$ the \emph{QIB-perfect reformulation} of $q$ w.r.t. $\T$ and $\pol$. 

We prove FO-rewritability of entailment of BCQs under QIB censors in $\dlliter$ by exploiting a correspondence %that exists 
between this problem and entailment of BCQs under IAR-semantics for DL ontologies, which is indeed FO-rewritable for $\dlliterden$, i.e., $\dlliter$ enriched with denial assertions~\cite{LLRRS15}. We recall that the IAR-semantics is an inconsistency-tolerant semantics that allows for meaningful entailment also when the ABox contradicts the TBox of an ontology (for instance, when $\A=\{A(d),B(d),C(d)\}$ and $\T=\{A\ISA \NOT B\}$). The IAR-semantics is based on the notion of \emph{ABox repair} ($A$-repair), which is a maximal subset of the ABox that is consistent with the TBox (in our example there are two $A$-repairs, $\R_1=\{A(d),C(d)\}$ and $\R_2=\{B(d),C(d)\}$). Then, entailment under IAR-semantics is defined as follows: let $\T$ be a $\dlliterden$ TBox, $\A$ be an ABox, and $q$ be a BCQ, $\iarEntailment(\T,\A,q)$ is the problem of verifying whether $\T \cup \R_{iar}\models q$, where $\R_{iar}$ is the intersection of all A-repairs of $\O=\T \cup \A$, called the IAR-repair of $\O$ (in our example, $\R_{iar}=\{C(d)\}$).

To establish the relationship between QIB-entailment and IAR-entailment, we define secrets, which play in our framework a role similar to minimal inconsistent sets in inconsistency-tolerant query answering~\cite{LLRRS15}.

Let $\T$ be a TBox, let $\pol$ be a policy, and let $\A$ be an ABox. We say that a set of ABox assertions $\S \subseteq \cll{\GA}{\A}$ is a \textit{secret} in $\T \cup \pol \cup \A$, if $\T \cup \pol \cup \S$ is inconsistent and for each assertion $\sigma \in \S$ we have that $\T \cup \pol \cup \S \setminus \{ \sigma \}$ is consistent. We denote by $\secrets(\T,\pol,\A)$ the set of all the secrets in  $\T \cup \pol \cup \A$. 

We now provide the following key property. 

\begin{lemma}\label{lem:no-secrets}	
	Let $\T$ be a DL TBox, let $\pol$ be a policy, let $\A$ be an ABox, and let $q$ be a BCQ. $\intEntailment(\T,\P,\A,q)$ is true iff there exists a $\A' \subseteq \cll{\GA}{\A}$ such that:
	\begin{itemize} \itemsep 0pt
		%\item[$(i)$] $\T \cup \pol \cup \A'$ is consistent;
		\item[$(i)$] $\T \cup \A' \models q$;
		%\item[$(ii)$] there is no secret $\S \in\secrets(\T,\pol,\A)$ such that $\A' \cap \S \neq \emptyset$.
		\item[$(ii)$] $\A' \cap \S = \emptyset$, for each secret $\S \in\secrets(\T,\pol,\A)$. 
	\end{itemize}  
\end{lemma}
\begin{proof}
  ($\Leftarrow$). We first show that given an ABox assertion $\alpha \in \cll{\GA}{\A}$, there exists an optimal IB censor $\cens \in \stcenset{\T,\P}$ such that $\alpha \not \in \cens(\A)$ only if there exists a secret $\S$ in $\secrets(\T,\pol,\A)$ such that $\alpha \in \S$.  Suppose, by way of contradiction, that $\alpha$ does not belong to any secret in $\secrets(\T,\pol,\A)$. This means that $\cens(\A) \cup \P \cup \{\alpha \}$ is still consistent
%% \nb{perche'?segue in modo cosi' ovvio?}
  and so $\cens$ is not optimal, from which the contradiction. Now, suppose that there exists an ABox $\A' \subseteq \cll{\GA}{\A}$ such that: $(i)$ $\T \cup \A' \models q$; and $(ii)$ there is no secret $\S$ in $\secrets(\T,\pol,\A)$ such that $\A' \cap \S \neq \emptyset$. From what shown above and from condition $(ii)$, we have that $\A' \subseteq \cens(\A)$ for every $\cens \in \stcenset{\T,\P}$. This means that $\A' \subseteq \qibcens_{\T,\pol}(\A)$. %$\A' \subseteq \stcens'_{\T,\pol}(\A)$ for each QIB censor $\stcens'_{\T,\pol}$  for $\T$ and $\pol$. 
	Moreover, since $\T \cup \A' \models q$, we have that %$\T \cup \stcens'_{\T,\pol}(\A) \models q$, 
	$\T \cup \qibcens_{\T,\pol}(\A) \models q$, which shows the thesis.   

	($\Rightarrow$). Suppose that $\intEntailment(\T,\P,\A,q)$ is true. This means that 
	%there exists a QIB censor $\stcens_{\T,\pol}$ for $\T$ and $\pol$ such that $\T \cup \stcens_{\T,\pol}(\A) \models q$. 
	$q \in \qibcens_{\T,\pol}(\A)$. Since %$\stcens_{\T,\pol}$ 
	$\qibcens_{\T,\pol}$ is an IB censor, then there exists an ABox $\A'$ such that  %$\stcens_{\T,\pol}(\A) = \stcens_{\T,\pol}(\A'')$ 
	$\qibcens_{\T,\pol}(\A) = \qibcens_{\T,\pol}(\A')$ 
	and $\T \cup \P \cup \A'$ is consistent (that is, $\A$ and $\A'$ are indistinguishable w.r.t.\ $\qibcens_{\T,\pol}$). %$\stcens_{\T,\pol}$). 
	Hence, $\T \cup \A' \models q$. Moreover %$\A' \subseteq \stcens_{\T,\pol}(\A)$ 
	$\A' \subseteq \qibcens_{\T,\pol}(\A)$ and thus $\A' \subseteq \cll{\GA}{\A}$. So, $\A'$ satisfies condition $(i)$ of the lemma. As for condition $(ii)$ we proceed towards a contradiction. Suppose that there exists an ABox assertion $\alpha \in \A'$ and a secret $\S \in \secrets(\T,\pol,\A)$ such that $\alpha \in  \S$. From Definition~\ref{def:qos-censor}, we have that $\A' \subseteq \cens'(\A)$ for every $\cens' \in \stcenset{\T,\P}$, and so, $\alpha \in \cens'(\A)$ for every $\cens' \in \stcenset{\T,\P}$. Since $\S \setminus \{ \alpha \}$ is consistent with $\T \cup \pol$, we have that  $\S \setminus \{ \alpha \}$ is not a secret in $\T \cup \pol \cup \A$. So it is possible to define an optimal IB censor whose theory contains $\S \setminus \{ \alpha \}$,
%% \nb{e' cosi' ovvio?}
which is a contradiction, and so $\A'$ satisfies condition $(ii)$ too. 
\end{proof}

The following theorem establishes the relationship between QIB-entailment and IAR-entailment.

\begin{theorem}
 \label{thm-iar-qos}
%Let $\T$ be a $\dlliter$ TBox and $\A$ be an ABox such that $\T\cup\A$ is consistent\nb{Per CQE dobbiamo sempre assumere questa consistenza},  and let $q$ be a BCQ. Then $\intEntailment(\T,\P,\A,q)$ is $true$ iff $\iarEntailment(\T \cup \P,\cll{\GA}{\A},q)$ is true.
Let $\T$ be a $\dlliter$ TBox, let $\pol$ be a policy, let $\A$ be an ABox, and let $q$ be a BCQ. $\intEntailment(\T,\P,\A,q)$ is true iff $\iarEntailment(\T \cup \P,\cll{\GA}{\A},q)$ is true.
\end{theorem}
\begin{proof}
  Since $\T\cup\A$ is consistent, then the secrets in $\T\cup\P\cup\cll{\GA}{\A}$ coincide with the minimal subsets of $\cll{\GA}{\A}$ that are inconsistent with $\T\cup\P$. Therefore, the IAR-Repair $\R$ of $\tup{\T\cup\P,\cll{\GA}{\A}}$ is the set of ground atoms from $\cll{\GA}{\A}$ that do not belong to any secret in $\T\cup\P\cup\cll{\GA}{\A}$. Thus, from Lemma \ref{lem:no-secrets} the thesis follows.
%% it follows that $\cll{\CQL}{\R}$ is equal to the theory of the QIB censor for $\T$ and $\P$ over $\cll{\GA}{\A}$, i.e., $\cll{\CQL}{\R}=\stcens_{\T,\pol}(\cll{\GA}{\A})$, where $\stcens_{\T,\pol}$ is the QIB censor for $\T$ and $\P$. Consequently, the thesis follows.
\end{proof}

Theorem~\ref{thm-iar-qos} actually states that, to solve QIB-entailment, we can resort to the query rewriting techniques used to establish IAR-entailment given in \cite{LLRRS15}, provided that we compute $\cll{\GA}{\A}$.
We recall that query entailment under IAR-semantics in a DL $\L$ is FO-rewritable, if for every TBox $\T$ expressed in $\L$ and every BCQ $q$, one can effectively compute an FO query $q_r$ such that for every ABox $\A$, $\iarEntailment(\T,\A,q)$ is true iff $\A \models q_r$. The query $q_r$ is called the \emph{IAR-perfect reformulation} of $q$ w.r.t. $\T$. 

To establish FO-rewritability of QIB-entailment in $\dlliter$, however, we still need to address the above mentioned computation of $\cll{\GA}{\A}$, and turn it into an additional query reformulation step. To this aim, we can exploit the fact that, for a $\dlliterden$ ontology $\T \cup \A$, an FO query $q$ evaluates to true over $\cll{\GA}{\A}$  iff $q'$ evaluates to true over $\A$, where $q'$ is obtained by suitably rewriting each atom of $q$ according to the positive inclusions of $\T$. Intuitively, in this way we cast into the query all the possible causes of the facts that are contained in the closure of the ABox w.r.t.\ the TBox (similarly to what is done in query rewriting algorithms for \dllite \cite{CDLLR07}).
%%
%% To compute such $q'$, we use the function $\atomrewr(q,\T)$, which substitutes each atom $\alpha=X(\vec{t})$ of $q$, with $X$ either an atomic concept or an atomic role and $\vec{t}$ the sequence of its terms (either variables or constants), with $(\beta)$, computed as follows

%% \[
%% \begin{array}{l}
%%  \beta=X(\vec{t})\\
%%  \textrm{for each inclusion } Y \ISA X \textrm{ in } \T^*\\
%%  ~~~~~~\beta=\beta \lor Y(\vec{t})
%% \end{array} 
%% \]
%% where $\T^*$ denotes the deductive closure of $\T$.

%% For example, if $\T=\{A\ISA C, B \ISA C\}$ and $q=\exists x,y \per C(x) \land P(x,y)$\nb{controllare che la sintassi della query sia in linea con i preliminari}, then  
%% $\atomrewr(q,\T)$ returns the query $q=\exists x,y \per (C(x) \lor A(x) \lor B(x)) \land P(x,y)$.

To compute such a query $q'$, we use the function $\atomrewr(q,\T)$, which substitutes each atom $\alpha$ of $q$ with the formula $\phi(\alpha)$ defined as follows (where $A,B$ are atomic concepts and $R,S$ are atomic roles):
%% \begin{tabbing}
%% ~\=$\phi(A(t))=$\=$\bigvee_{\T\models B \ISA A}B(t) \vee \bigvee_{\T\models \exists R \ISA A}(\exists x. R(t,x))\vee$\\
%% 	\>                    \>$\bigvee_{\T\models \exists R^{-} \ISA A}(\exists x. R(x,t)) $\\
%% 	\\[-4px]
%% \>	$\phi(R(t_1,t_2))=\bigvee_{\T\models S \ISA R}S(t_1,t_2) \vee \bigvee_{\T\models S^{-} \ISA R}S(t_2,t_1).$
%% \end{tabbing}

\vspace*{-2mm}

\begin{small}
\[
\begin{array}{l}
  \phi(A(t))=\bigvee_{\T\models B \ISA A}B(t) \vee \bigvee_{\T\models \exists R \ISA A}(\exists x. R(t,x))\vee\\[2mm]
  \qquad\qquad\;\;\; \bigvee_{\T\models \exists R^{-} \ISA A}(\exists x. R(x,t)) \\
\\
\phi(R(t_1,t_2))=\bigvee_{\T\models S \ISA R}S(t_1,t_2) \vee \bigvee_{\T\models S^{-} \ISA R}S(t_2,t_1)
\end{array}
\]
\end{small}

For example, if $\T=\{A\ISA C, B \ISA C\}$ and $q=\exists x,y \per C(x) \land P(x,y)$, then  
$\atomrewr(q,\T)$ returns the query $q=\exists x,y \per (C(x) \lor A(x) \lor B(x)) \land P(x,y)$.

The following lemma, whose proof can be immediately obtained from the definitions of $\cll{\GA}{\cdot}$ and $\atomrewr(\cdot,\cdot)$, states the property we are looking for.
\begin{lemma}
\label{lem-atomrewr}
%Let $\O=\T \cup \A$ be a $\dlliterden$ ontology and $q$ be a FOL sentence. Then $\eval(q,\cll{\GA}{\A})=\eval(\atomrewr(q,\T),\A)$.
%% \nb{introdurre nei preliminari: with $\eval(q,\A)$, we denote the evaluation of a query $q$ over (the model isomorphic to) an ABox $\A$}
Let $\T$ be a $\dlliterden$ TBox, let $\A$ be an ABox, and let $q$ be an FO sentence. Then $\eval(q,\cll{\GA}{\A})=\eval(\atomrewr(q,\T),\A)$.
\end{lemma}

We are now able to extablish FO-rewritability of QIB-entailment.

\begin{theorem}
Let $\T$ be a $\dlliter$ TBox, let $\pol$ be a policy, let $q$ be a BCQ, and let $q_r$ be an FO sentence that is a IAR-perfect reformulation of $q$ w.r.t.\ the $\dlliterden$ TBox $\T \cup \pol$. Then, the FO sentence $\atomrewr(q_r,\T)$ is a QIB-perfect reformulation of $q$ w.r.t.\ $\T$ and $\P$.
\end{theorem}
\begin{proof}
Let the FO sentence $q_r$ be an IAR-perfect reformulation of $q$ w.r.t.\ the $\dlliterden$ TBox $\T \cup \pol$. Then, for every ABox $\A$, $\iarEntailment(\T \cup \P,\cll{\GA}{\A},q)$ is true iff $\eval(q_r,\cll{\GA}{\A})$ is true. Now, from Lemma \ref{lem-atomrewr}, it follows that, for every ABox $\A$, $\eval(q_r,\cll{\GA}{\A})=\eval(\atomrewr(q_r,\T),\A)$. And since by Theorem \ref{thm-iar-qos}, for every ABox $\A$ such that $\T\cup\A$ is consistent, $\iarEntailment(\T \cup \P,\cll{\GA}{\A},q)$ is true iff $\intEntailment(\T,\P,\A,q)$ is true, it follows that the FO sentence $\atomrewr(q_r,\T)$ is a QIB-perfect reformulation of $q$ w.r.t.\ $\T$ and $\P$.
\end{proof}

Since IAR-entailment is actually FO rewritable, as shown in \cite{LLRRS15}, the above theorem proves the FO rewritability of QIB-entailment for $\dlliter$ TBoxes. Moreover, the above theorem identifies a technique for obtaining the QIB-perfect reformulation of a CQ, based on a simple combination of the IAR-perfect reformulation algorithm of \cite{LLRRS15} and the $\atomrewr$ reformulation defined above.
Therefore:
%% The following corollary is thus immediate.

\begin{corollary}
Let $\T$ be a $\dlliter$ TBox, let $\pol$ be a policy, let $\A$ be an ABox, and let $q$ be a BCQ. The problem $\intEntailment(\T,\P,\A,q)$ is in $\ACz$ in data complexity.
\end{corollary}

	\section{Conclusions}
\label{sec:conclusions}

%%Elimina qui
In this paper we have studied the approach to controlled query evaluation based on instance indistinguishability: we have applied this approach to Description Logic ontologies, we have studied its relationship with another confidentiality-preserving approach, and we have established complexity results for this form of controlled query evaluation in the case of $\dlliter$ ontologies.

%In this paper we have identified...
Notably, in this framework we have identified a tractable and semantically well-founded notion of CQE that enjoys the first-order rewritability property. We believe that this result opens the way towards practical implementations of CQE engines for DL ontologies and Ontology-based Data Access. We are currently working to achieve this goal.

%% Another important future direction is a deeper study of the user/attacker model. The present framework inherits from its predecessors a relatively simple model in which only the \emph{deductive} abilities of the user are considered: we think that such a model of the user should be enriched to capture more realistic data protection scenarios.

%% Another important future direction is a deeper study of the user/attacker model. The present framework inherits from its predecessors a relatively simple model, in which the user is assumed to know (at most) the TBox and all the answers to queries returned by the system, and only the \emph{deductive} abilities of the user over such knowledge are considered. We think that such a user model might need to be enriched to capture more realistic data protection scenarios.

%% Another important future direction is a deeper study of the user/attacker model. The present framework inherits from its predecessors a relatively simple model in which only the \emph{deductive} abilities of the user are considered: we think that such a model of the user should be enriched to capture more realistic data protection scenarios.

Another important future direction is a deeper study of the user model. Our framework inherits from its predecessors a relatively simple model, which assumes that the user knows (at most) the TBox and all the query answers returned by the system, and considers only the \emph{deductive} abilities of the user over such knowledge. This user model might need to be enriched to capture more realistic data protection scenarios.

%%% Local Variables:
%%% mode: latex
%%% TeX-master: "main"
%%% save-place: t
%%% End:

\section * {Acknowledgements} 
This work was partly supported by EU within the H2020 under grant agreement 834228 (ERC Advanced Grant WhiteMec) and under grant agreement  825333 (MOSAICrOWN), by Regione Lombardia within the Call Hub Ricerca e Innovazione under grant agreement 1175328 (WATCHMAN), and by Sapienza Università di Roma (2019 project CQEinOBDM).  
	
	\bibliographystyle{named}
	\bibliography{medium-string,krdb,w3c,local-bib}

\begin{thebibliography}{}

\bibitem[\protect\citeauthoryear{Artale \bgroup \em et al.\egroup
  }{2009}]{ACKZ09}
Alessandro Artale, Diego Calvanese, Roman Kontchakov, and Michael
  Zakharyaschev.
\newblock The \textit{DL-Lite} family and relations.
\newblock {\em J.\ of Artificial Intelligence Research}, 36:1--69, 2009.

\bibitem[\protect\citeauthoryear{Baader \bgroup \em et al.\egroup
  }{2007}]{BCMNP07}
Franz Baader, Diego Calvanese, Deborah McGuinness, Daniele Nardi, and Peter~F.
  Patel-Schneider, editors.
\newblock {\em The Description Logic Handbook: {T}heory, Implementation and
  Applications}.
\newblock Cambridge University Press, 2nd edition, 2007.

\bibitem[\protect\citeauthoryear{Benedikt \bgroup \em et al.\egroup
  }{2018}]{BeCK18}
Michael Benedikt, Bernardo Cuenca~Grau, and Egor~V. Kostylev.
\newblock Logical foundations of information disclosure in ontology-based data
  integration.
\newblock {\em Artificial Intelligence}, 262:52--95, 2018.

\bibitem[\protect\citeauthoryear{Benedikt \bgroup \em et al.\egroup
  }{2019}]{BBJT19}
Michael Benedikt, Pierre Bourhis, Louis Jachiet, and Micha{\"{e}}l Thomazo.
\newblock Reasoning about disclosure in data integration in the presence of
  source constraints.
\newblock In {\em Proc.\ of the 28th Int.\ Joint Conf.\ on Artificial
  Intelligence (IJCAI)}, pages 1551--1557, 2019.

\bibitem[\protect\citeauthoryear{Biskup and Bonatti}{2004a}]{BiBo04}
Joachim Biskup and Piero~A. Bonatti.
\newblock Controlled query evaluation for enforcing confidentiality in complete
  information systems.
\newblock {\em Int.\ J.\ of Information Security}, 3(1):14--27, 2004.

\bibitem[\protect\citeauthoryear{Biskup and Bonatti}{2004b}]{BiskupB04}
Joachim Biskup and Piero~A. Bonatti.
\newblock Controlled query evaluation for known policies by combining lying and
  refusal.
\newblock {\em Ann.\ of Mathematics and Artificial Intelligence},
  40(1-2):37--62, 2004.

\bibitem[\protect\citeauthoryear{Biskup and Bonatti}{2007}]{BiskupBonatti07}
Joachim Biskup and Piero~A. Bonatti.
\newblock Controlled query evaluation with open queries for a decidable
  relational submodel.
\newblock {\em Ann.\ of Mathematics and Artificial Intelligence},
  50(1--2):39--77, 2007.

\bibitem[\protect\citeauthoryear{Biskup and Weibert}{2008}]{BiskupW08}
Joachim Biskup and Torben Weibert.
\newblock Keeping secrets in incomplete databases.
\newblock {\em Int.\ J.\ of Information Security}, 7(3):199--217, 2008.

\bibitem[\protect\citeauthoryear{Biskup}{2000}]{Biskup00}
Joachim Biskup.
\newblock For unknown secrecies refusal is better than lying.
\newblock {\em Data and Knowledge Engineering}, 33(1):1--23, 2000.

\bibitem[\protect\citeauthoryear{Bonatti and Sauro}{2013}]{BoSa13}
Piero~A. Bonatti and Luigi Sauro.
\newblock A confidentiality model for ontologies.
\newblock In {\em Proc.\ of the 12th Int.\ Semantic Web Conf.\ (ISWC)}, volume
  8218 of {\em Lecture Notes in Computer Science}, pages 17--32, 2013.

\bibitem[\protect\citeauthoryear{Bonatti \bgroup \em et al.\egroup
  }{1995}]{BonattiKS95}
Piero~A. Bonatti, Sarit Kraus, and V.~S. Subrahmanian.
\newblock Foundations of secure deductive databases.
\newblock {\em {IEEE} Trans. Knowl. Data Eng.}, 7(3):406--422, 1995.

\bibitem[\protect\citeauthoryear{Calvanese \bgroup \em et al.\egroup
  }{2007}]{CDLLR07}
Diego Calvanese, Giuseppe De~Giacomo, Domenico Lembo, Maurizio Lenzerini, and
  Riccardo Rosati.
\newblock Tractable reasoning and efficient query answering in description
  logics: The \textit{DL-Lite} family.
\newblock {\em J.\ of Automated Reasoning}, 39(3):385--429, 2007.

\bibitem[\protect\citeauthoryear{Calvanese \bgroup \em et al.\egroup
  }{2012}]{CDLR12}
Diego Calvanese, Giuseppe De~Giacomo, Maurizio Lenzerini, and Riccardo Rosati.
\newblock View-based query answering in description logics: Semantics and
  complexity.
\newblock {\em J.\ of Computer and System Sciences}, 78(1):26--46, 2012.

\bibitem[\protect\citeauthoryear{Cuenca~Grau and Horrocks}{2008}]{GrHo08}
Bernardo Cuenca~Grau and Ian Horrocks.
\newblock Privacy-preserving query answering in logic-based information
  systems.
\newblock In {\em Proc.\ of the 18th Eur.\ Conf.\ on Artificial Intelligence
  (ECAI)}, pages 40--44, 2008.

\bibitem[\protect\citeauthoryear{Cuenca~Grau \bgroup \em et al.\egroup
  }{2013}]{GKKZ13}
Bernardo Cuenca~Grau, Evgeny Kharlamov, Egor~V. Kostylev, and Dmitriy
  Zheleznyakov.
\newblock Controlled query evaluation over {OWL} 2 {RL} ontologies.
\newblock In {\em Proc.\ of the 12th Int.\ Semantic Web Conf.\ (ISWC)}, volume
  8218 of {\em Lecture Notes in Computer Science}, pages 49--65, 2013.

\bibitem[\protect\citeauthoryear{Cuenca~Grau \bgroup \em et al.\egroup
  }{2015}]{GKKZ15}
Bernardo Cuenca~Grau, Evgeny Kharlamov, Egor~V. Kostylev, and Dmitriy
  Zheleznyakov.
\newblock Controlled query evaluation for datalog and {OWL} 2 profile
  ontologies.
\newblock In {\em Proc.\ of the 24th Int.\ Joint Conf.\ on Artificial
  Intelligence (IJCAI)}, pages 2883--2889, 2015.

\bibitem[\protect\citeauthoryear{Lembo \bgroup \em et al.\egroup
  }{2015}]{LLRRS15}
Domenico Lembo, Maurizio Lenzerini, Riccardo Rosati, Marco Ruzzi, and
  Domenico~Fabio Savo.
\newblock Inconsistency-tolerant query answering in ontology-based data access.
\newblock {\em J.\ of Web Semantics}, 33:3--29, 2015.

\bibitem[\protect\citeauthoryear{Lembo \bgroup \em et al.\egroup
  }{2019}]{LeRS19}
Domenico Lembo, Riccardo Rosati, and Domenico~Fabio Savo.
\newblock Revisiting controlled query evaluation in description logics.
\newblock In {\em Proc.\ of the 28th Int.\ Joint Conf.\ on Artificial
  Intelligence (IJCAI)}, pages 1786--1792, 2019.

\bibitem[\protect\citeauthoryear{Motik \bgroup \em et al.\egroup
  }{2012}]{W3Crec-OWL2-Profiles}
Boris Motik, Bernardo Cuenca~Grau, Ian Horrocks, Zhe Wu, Achille Fokoue, and
  Carsten Lutz.
\newblock {OWL~2} {W}eb {O}ntology {L}anguage profiles (second edition).
\newblock {W3C} {R}ecommendation, World Wide Web Consortium, December 2012.
\newblock Available at \protect\url{http://www.w3.org/TR/owl2-profiles/}.

\bibitem[\protect\citeauthoryear{Sicherman \bgroup \em et al.\egroup
  }{1983}]{SiJR83}
George~L. Sicherman, Wiebren de~Jonge, and Reind~P. van~de Riet.
\newblock Answering queries without revealing secrets.
\newblock {\em {ACM} Trans. Database Syst.}, 8(1):41--59, 1983.

\bibitem[\protect\citeauthoryear{Studer and Werner}{2014}]{StWe14}
Thomas Studer and Johannes Werner.
\newblock Censors for boolean description logic.
\newblock {\em Trans. Data Privacy}, 7(3):223--252, 2014.

\bibitem[\protect\citeauthoryear{Tao \bgroup \em et al.\egroup }{2014}]{TaSH14}
Jia Tao, Giora Slutzki, and Vasant~G. Honavar.
\newblock A conceptual framework for secrecy-preserving reasoning in knowledge
  bases.
\newblock {\em ACM Trans.\ on Computational Logic}, 16(1):3:1--3:32, 2014.

\bibitem[\protect\citeauthoryear{Xiao \bgroup \em et al.\egroup
  }{2018}]{XiaoCKLPRZ18}
Guohui Xiao, Diego Calvanese, Roman Kontchakov, Domenico Lembo, Antonella
  Poggi, Riccardo Rosati, and Michael Zakharyaschev.
\newblock Ontology-based data access: {A} survey.
\newblock In {\em Proc.\ of the 27th Int.\ Joint Conf.\ on Artificial
  Intelligence (IJCAI)}, pages 5511--5519, 2018.

\end{thebibliography}
\end{document}